\documentclass[runningheads]{llncs}
\usepackage[T1]{fontenc}
\usepackage{graphicx}
\usepackage{algorithm}
\usepackage{algorithmic}
\usepackage[colorinlistoftodos]{todonotes}
\usepackage{amsmath}
\usepackage{dsfont}
\usepackage{amsfonts}
\usepackage{comment}
\usepackage[]{hyperref}
\usepackage[backend=biber,style=numeric, sorting=none]{biblatex}
\begin{document}
\title{Hierarchical Variable Importance with Statistical Control for Medical Data-Based Prediction
%\todo{previous title: Hierarchical Variable Importance Measurement with Statistical Control for Highly Correlated Neuroimaging Data}
}

\titlerunning{Hierarchical CPI}
% If the paper title is too long for the running head, you can set
% an abbreviated paper title here
%
\author{Joseph Paillard\inst{1, 2} \and Antoine Collas\inst{2} \and Denis A. Engemann\inst{1} \and \\ Bertrand Thirion\inst{2}}
%index{Paillard, Joseph}
%index{Collas, Antoine}
%index{Engemann, Denis}
%index{Thirion, Bertrand}

%
\authorrunning{J. Paillard et al.}
% First names are abbreviated in the running head.
% If there are more than two authors, 'et al.' is used.
%
\institute{Roche Pharma Research \& Early Development, F. Hoffmann-La Roche Ltd, \\ Basel, Switzerland \\ 
\and Université Paris-Saclay, Inria, CEA, Paris, Palaiseau, France \\
Correspondance: \email{joseph.paillard@roche.com}\\
}
\maketitle              % typeset the header of the contribution
\begin{abstract}
% Recent advances in machine learning have greatly expanded the repertoire of medical imaging methods. 
% %
% However, powerful predictive models are difficult to interpret, which limits their utility for medical applications. 
% %
% In particular, the development of statistically controlled variable-importance procedures that accommodate complex architectures beyond linear models remains an active area of research. 
% %
% Although several methods have been proposed to control the type-1 error rate, they suffer from a lack of power when dealing with highly correlated data, a common problem in medical imaging. 
% %
% Here, we propose a variable importance estimation algorithm that combines hierarchical clustering and conditional importance to measure variable importance in high-dimensional and highly correlated data without requiring prior knowledge of the variable organization. 
% %
% We compare this new approach with state-of-the-art methods to assess its statistical guarantees and accuracy. 
% %
% We then demonstrate the effectiveness of the method on two real-world applications: classification of dementia diagnosis from MRI data (ADNI dataset) and classification of eye status from EEG (TDBRAIN dataset).
%
Recent advances in machine learning have greatly expanded the repertoire of predictive methods for medical imaging.
However, the interpretability of complex models remains a challenge, which limits their utility in medical applications.  
Recently, model-agnostic methods have been proposed to measure conditional variable importance and accommodate complex non-linear models. 
However, they often lack power when dealing with highly correlated data, a common problem in medical imaging.
We introduce Hierarchical-CPI, a model-agnostic variable importance measure that frames the inference problem as the discovery of groups of variables that are jointly predictive of the outcome.
By exploring subgroups along a hierarchical tree, it remains computationally tractable, yet also enjoys explicit family-wise error rate control.
%
% provides richer insights than traditional variable-level measures
%
% We introduce Hierarchical-CPI, a model-agnostic variable importance measure with explicit family-wise error rate control. 
% Hierarchical-CPI measures the importance of variable subgroups along a hierarchical tree, which provides richer insights than traditional variable-level importance measures while remaining computationally tractable.
%
Moreover, we address the issue of vanishing conditional importance under high correlation with a tree-based importance allocation mechanism.
We benchmarked Hierarchical-CPI against state-of-the-art variable importance methods. 
Its effectiveness is demonstrated in two neuroimaging datasets: classifying dementia diagnoses from MRI data (ADNI dataset) and analyzing the Berger effect on EEG data (TDBRAIN dataset), identifying biologically plausible variables.

\keywords{Statistics, neuroimaging, interpretable machine learning}
\end{abstract}
\section{Introduction}
Within the field of medcial imaging, machine learning holds great promise to facilitate prediction of clinical outcomes, see e.g. \cite{wen_2020_convolutional,tosun2024identifying,gemein_2020_machine, cuingnet_2011_automatic, abraham_2014_machine,tousignant2019prediction}. 
However, these advances have also opened major interpretability challenges.
A key issue is how to infer the importance of features from prediction models going beyond ordinary least squares to accommodate a large number of predictors and represent non-linear associations between features and outcomes. 
Therefore, developing methods to measure variable importance in a model-agnostic manner is critical in order to obtain clinical insights and develop biomarkers,  
%
% Such methods have the potential to facilitate biological discoveries. 
%
for instance, using brain images for the diagnosis of Alzheimer Disease (AD) based on existing cohorts, 
% such as the Alzheimer's Disease Neuroimaging Initiative (ADNI)
or data from clinical trials \cite{petersen_2010_alzheimer, samper_2018_reproducible,tosun2024identifying,tousignant2019prediction}.
%
% Or to identify prognostic biomakers helping to understand progression mechanisms and could ultimately translate into actionable medical insights.
% 
However, to develop trustworthy methods, it is essential to understand their theoretical guarantees, particularly concerning the risk of making false discoveries, which can be captured by the Family-Wise Error Rate (FWER) (see e.g. \cite{mandozzi_2016_hierarchical, Chevalier_2018_statistical}). 
Only few variable importance methods give access to such guarantees.
Moreover, we focus here on \emph{conditional importance}, meaning the importance measure whether a variable is \emph{directly} predictive of the outcome, without being explained away by other variables \cite{sobol_2001_Global,chamma_2024_variable}. 
Such conditional importance is needed to establish that a marker carries independent information about the outcome,  rather than merely reflecting distributed factors that are also present in other variables.
Conditional importance analysis is particularly difficult in datasets that exhibit strong correlation structures such as image- or genomics-based biomarkers, or health data that reflect common latent factors \cite{chevalier_2021_decoding}.
%
% Robust empirical validation is also necessary to ensures the practical applicability and reliability of these methods.

We assess the face validity of the approach with two tasks that have been extensively studied in the literature---the effect of AD on structural MRI and the Berger effect on electroencephalography (EEG)---to allow for a form of confirmation, addressing the challenge posed by the absence of ground truth in variable importance methods.

% \textbf{XXX miss an explicit motivation for **conditional** importance: take inspiration from Ahmad's works.}

\subsection{Related Work}
This work focuses \emph{global} variable importance, as opposed to local variable importance methods such as \textit{LIME} \cite{Ribeiro_2016_why} or \textit{SHAP} \cite{Strumbelj_2014_explaining}.
Global variable importance is estimated using methods such as global sensitivity analysis \cite{sobol_2001_Global} or the popular Leave One Covariate Out (LOCO) approach \cite{Toshimitsu_1996_importance,Williamson_2023_general}.
These methods can accommodate different types of learners, taking advantage of advances in machine learning to measure importance in complex and nonlinear models \cite{Williamson_2023_general, verdinelli_2024_feature}.
Similarly, conditional permutation approaches have been shown to estimate a quantity equivalent to LOCO at a cheaper computational cost and with a faster convergence rate \cite{chamma_2024_statistically}.
These methods have in common to provide a good control of the type-1 error rate, that is considering a null variable (or group) as important.  
However, all approaches suffer from an inherent limitation: conditional importance decreases as correlation increases.
For instance, considering two random normal variables $X_1, X_2$ with correlation $\rho$  in a simple linear model $y=\beta_1 X_1 + \beta_2 X_2$ the importance of $X_1$ decreases proportionally to $(1-\rho^2)$.
% This is a core issue with conditional inference.
%
% Consequently, when considering highly correlated variables, which is typically the case in neuro-imaging, the importance of individual variable will vanish due to the correlation structure. 
% \todo{Why adding X2 ? My point was that even if only X1 is predictive, it's conditional importance decreases because of unimportant variables}\todo{Bertrand: because otherwise you don't see the problem of vanishing importance}

To mitigate this issue, methods based on variables grouping have been proposed to identify groups of highly-correlated, hence indistinguishable variables that predict the outcome \cite{chamma_2024_variable, chevalier_2021_decoding}.
Variable grouping can be performed based on prior knowledge about the data or by using clustering techniques.
While this effectively increases the statistical power by averaging correlated variables, it also reduces the precision in the sense that error control only holds at the group level.
When performing the grouping in a data-driven way, choosing the clustering scheme and parameters has a critical impact yet has no obvious solution. 
%
% Indeed, the groups of correlated variables can have very different scale, for instance when dealing with multimodal data where the number of features per modality differs from order of magnitude. 
% BT: you can always scale the variables
%
A line of work relying on linear models and agglomerative clustering has been proposed in that direction with applications to genomics data \cite{mandozzi_2016_hierarchical, binh_2019_ecko, Chevalier_2018_statistical}. 
Agglomerative clustering offers a compelling solution as it naturally explores different groupings at various resolutions along the hierarchical tree learned from the data.
However, this approach relied on Lasso regression
and would consequently limit the user in the choice of the model used to predict the outcome of interest from the variables. 
% , an approach which is no longer state-of-the-art and too restrictive.\todo{I don't find that a strong argument, it is still frequently used in R&D applications; the argument you commented out is more appropriate (DE)} 

Another popular model explainability approach is Shapley Additive Global importancE (\textit{SAGE}) \cite{covert_2020_sage}. 
Based on Shapley values, this approach estimates an importance score for a given variable by conditioning on all subgroups that include this variable.
% 
% This proceedure provides a much richer information than a single clustering and provides individual importance instead of cluster level values. 
This procedure provides a more nuanced view than strict conditional importance, because it decomposes additively the variance explained by the model into variable importance. 
However, it suffers from two main limitations. 
%
% First and foremost, this method are based on the SHAPLEY axioms which is a questionable choice. 
%
% As discussed in \cite{verdinelli_2024_feature}, the additivity property is particularly debatable as it is not clear why one would expect the importance of a pair of variables to be equal to the sum of individual importance's in the presence of interaction. 
%
% However, by distributing importance across correlated features, this approach can lead to an increased number of false discoveries.
First, as an aggregated statistic, it obcures the role of variables in the prediction function \cite{verdinelli_2024_feature}, and is unable to distinguish between a predictive variable and another, non-predictive variable yet correlated with a predictive one.
% 
% While this explainability is sufficient for certain use-cases, many applications, especially in medicine require stricter guarantees. 
% 
Second, the exploration of all submodels comes with an exponential explosion of computation cost, making this approach intractable. 
While implementations rely on Monte-Carlo sampling instead of exhausting the full combinatorial sum, the number of sampling steps needed to obtain accurate estimates still leads to intractable computation costs \cite{verdinelli_2024_feature}. 
This two limitations are clearly visible in our experiments in \autoref{fig:simulation}. 

Our contributions are \textit{i)} to introduce Hierarchical-CPI, a model-agnostic variable importance measure that improves FWER control; 
it explores subgroups in a tree-guided manner, using agglomerative clustering to provide more information than variable-level importance while remaining tractable;
\textit{ii)} to present an approach that enforces importance conservation through downstream importance allocation strategy, addressing the issue of vanishing importance under high correlation. % a common challenge in neuroimaging

\section{Methods}

\paragraph{Notations:} We denote $X$ as the variables, $y$ as the outcome, and $\mu$ as the predictive model. 
$X_G$ represents the set of variables belonging to a group $G$, and $X_{-G}$ denotes the set of variables in the complement of $G$.  
The importance of a group is denoted as $\psi_G$. We use $S^*$ to denote the support (or set of active variables) and $S_0$ for the set of null variables.  
In the hierarchical tree defined by the clustering, $P$ refers to a parent node, and $L$ and $R$ refer to its left and right child nodes, respectively. 

\subsection{Hierarchical-CPI}
We present a method for measuring variable importance while conditioning on others, with conditioning sets taken in a tree-organized hieracrchical representation of the variables. 
It balances \emph{precision} and \emph{statistical power}; precision refers to extracting the information located in groups of variables that are as small as possible;  statistical power is achieved by considering condition sets different from the set of all variables. 
\begin{algorithm}[h]
    \caption{Hierarchical CPI}
    \label{alg:hcpi}
    \textbf{Input}: $K$: number of folds, $\mu$: predictive model, $\nu$: imputation model, $(X, y)$: data\\
    \begin{algorithmic}[1] %[1] enables line numbers
    \STATE tree $\gets$ Fit hierarchical clustering on $X$ \\
    \FOR{k in $[1, \cdots, K]$} 
        \STATE $\hat \mu_k \gets$ Fit using $(X_{train}, y_{train})$ // fit the full model\\
        \FOR{node in tree} 
            \STATE G $\gets$ traversal(node)   // search variables belonging to the node\\
            \STATE $\hat{\nu}^k_G \gets \mathbb{E}[X^{train}_{G}| X^{train}_{-G}]$ // estimate the conditional distribution \\
            \STATE $\widetilde{X}_G^{test} \sim \hat{\nu}^k_G (X_{-G}^{test})$   // sample from the conditional distribution
            \STATE $\hat\psi^k_G \gets \mathcal{L}\left(y, \mu(\widetilde{X}^{test}_G)\right) - \mathcal{L}\left(y, \mu(X^{test})\right)$   // compute variable importance \\
        \ENDFOR
    \ENDFOR
\STATE $p_G$ $\gets$ t-test($\hat\psi^1_G, \cdots, \hat\psi^K_G$)   // compute p-value over folds\\
\STATE $p_G^h$ $\gets$ $\max_{G	\subseteq D} p_D^h$   // hierarchical adjustment
    \STATE $\textbf{return} \; p^h_{G_i}$ for $i=1, \cdots, 2p-1$
\end{algorithmic}
\end{algorithm}
The proposed method builds on top of Conditional Permutation Importance (CPI) \cite{chamma_2024_variable} which, given a group of variables $G$, a model $\mu$ and loss $\mathcal{L}$ estimates the conditional importance
%\textcolor{red}{
\begin{equation}
\label{eq:CPI}
    \psi_G= \mathcal{L}\left(y, \mu(\tilde{X}_G)\right) - \mathcal{L} \left(y, \mu(X) \right),
\end{equation}
where $\tilde{X}_G$ is obtained by substituting into the group $X_{G}$ variables sampled from the conditional distribution $(X_{G}|X_{-G})$ and leaving the $X_{-G}$ variables unchanged.
In brief, CPI quantifies the loss increase when conditioning on all other variables than those in  $G$.
%}
This approach estimates the well known total Sobol index \cite{sobol_2001_Global}.
In addition, hierarchical-CPI leverages Ward’s minimum variance method for agglomerative clustering to learn the hierarchical group structure \cite{ward_1963_hierarchical}.
The proposed method is presented in Algorithm \autoref{alg:hcpi}, for a problem with $p$ variables, it consists in  estimating the conditional permutation importance  of each group of variables within the hierarchical structure.
Empirical importance values are obtained in a K-fold cross-validation scheme, yielding $K$ estimates per group,  ($\psi_G^1, \cdots,\psi_G^K$).
A p-value $p_G$ is then derived based on a one-sample t-test.
Finally, the node-level p-values $p_G$ are hierarchically adjusted by,
\begin{align}
\label{eq:hierarchical_adjust}
    p_G^h = \max_{G\subseteq D} p_D,
\end{align}
to enforce that the p-value of a node is larger than the p-value of its parent.

\subsection{Hierarchical CPI achieves FWER control}
In this section, we demonstrate that the hierarchical-CPI approach controls the FWER under assumptions of estimator optimality and regularity.
The assumptions (A.1, A.2, B.1, B.2) from \cite{Williamson_2023_general} stipulate that the estimator $\mu$ must be optimal and exhibit sufficient regularity. 
These assumptions have been validated by independent work and are considered not too restrictive \cite{verdinelli_2024_feature, Williamson_2023_general, chamma_2024_statistically}.
We refer to a tree cut as a set of non-overlapping nodes within a hierarchical tree.
Let $S_0$ denote the set of groups that only contain null variables, and let $\hat S_\alpha = \{G \; | \; p_G \leq \alpha\}$ be the estimated set of active variables for a given significance level $\alpha \in [0, 1]$. 
The following result holds.
\begin{theorem}
% Under the assumption that the conditions (A.1, A.2, B.1, B.2) stated in \cite{Williamson_2023_general} on $\mu$, for any tree cut leading to $N_c$ clusters denoted $G_1, \cdots, G_{N_c}$, for any significance level $\alpha \in [0, 1]$ the multiplicity corrected p-values $\tilde p_{G}^h = \min(1, N_c\cdot p^h_{G})$ control the family-wise error rate at level $\alpha$ in the sense that, $\mathbb{P}(S_0 \cap \hat S \neq \emptyset) \leq \alpha$
Under the assumption that the conditions (A.1, A.2, B.1, B.2) stated in \cite{Williamson_2023_general} on $\mu$, for any significance level $\alpha \in [0, 1]$ the multiplicity corrected p-values $\tilde p_{G}^h = \min(1,C\cdot p^h_{G})$, with $C=p$ control the family-wise error rate at level $\alpha$, i.e. $\mathbb{P}(S_0 \cap \hat S_\alpha \neq \emptyset) \leq \alpha$. Where G is a node of a tree cut. 
\end{theorem}
\begin{proof}
\begin{align*}
     \mathbb{P}(S_0 \cap \hat{S}_\alpha \neq \emptyset) &= \mathbb{P}\left( \min_{G \subseteq S_0} { \tilde{p}_G^h \leq \alpha } \right) = \mathbb{P}\left( \bigcup_{G \subseteq S_0} p_G^h \leq \frac{\alpha}{C}\right)
\end{align*}
Then, given Boole's inequality, 
\begin{align*}
     \mathbb{P}(S_0 \cap \hat{S}_\alpha \neq \emptyset) & \leq\sum_{G \subset S_0} \mathbb{P}\left(p_G^h \leq \frac{\alpha}{C} \right) \\
     &\leq C\cdot \mathbb{P}\left(p_G^h \leq \frac{\alpha}{C} \right)
\end{align*}
Since a tree cut contains less than $p$ nodes, $\text{card}(\{G_i \;|\; G_i \subseteq S_0 \}_{i\in[1, C]}) \leq C$.\\
Furthermore, given that $p_G^h = \max_{G \subseteq D}\{ p_D \}$, where the maximum is taken over ancestor nodes, it comes that $\mathbb{P}\left( p_G^h \leq \alpha \right) \leq \mathbb{P}\left( p_G \leq \alpha \right)$. 
Finally, under assumptions (A.1, A.2, B.1, B.2), it has been shown in \cite{chamma_2024_statistically} that, $
\forall G \subset S_0, \;  \mathbb{P}\left( p_G \leq \alpha \right) \leq \alpha$. 
We then have, $\mathbb{P}(S_0 \cap \hat{S}_\alpha \neq \emptyset) \leq \alpha$
which completes the proof.
%\textbf{XXX but we need to  say that we are actually using 2p-1 groups}
\end{proof}
While this result holds when considering inference at the variable level, it is more general and applies to any node of the tree. 
Hierarchical CPI allows to learn a tree structure from the data and make inference at different levels. 

\subsection{Importance conservation to prevent importance vanishing}
\begin{figure}[h!]
\label{subsec:importance_conservation}
\includegraphics[width=\textwidth]{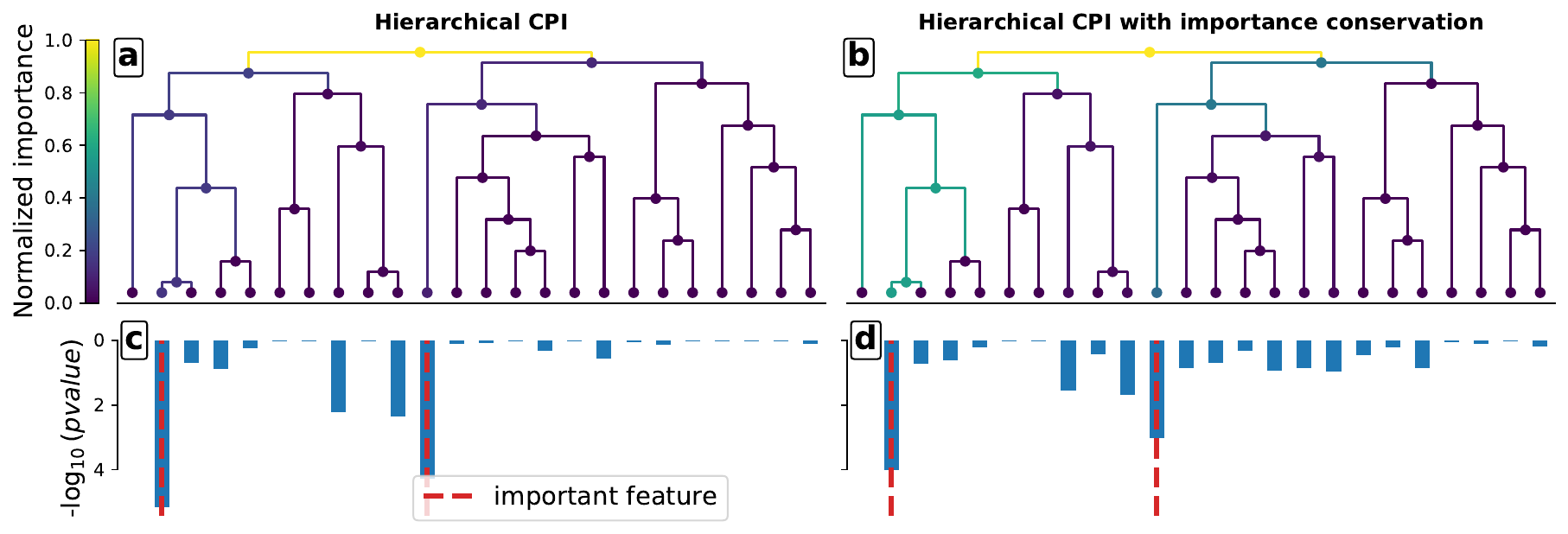}
\caption{
\textbf{Importance conservation prevents the importance from vanishing as the group size decreases in a high-correlation setting.} 
Example using simulated data with $n=300$ samples and $d=24$ variables.
The data is generated by blocks, each corresponding to an AR(1) with autocorrelation parameter
$\rho_{\text{max}}=0.95$.
\textbf{a} and \textbf{b} show the same dendrogram obtained through Ward's clustering.
Each node's color encodes the conditional importance of the variables it contains.
Without importance conservation, importance quickly vanishes down the tree.
\textbf{c} and \textbf{d} present the p-value distributions, demonstrating that both methods accurately rank important variables.
}
\label{fig:concept}
\end{figure} 
A common pitfall of conditional importance is that it vanishes under high correlation: For a parent node $P$ with two strongly correlated children nodes $L$ and $R$, then we have that $\psi_L^k + \psi_R^k \ll \psi_P^k$ for all $k$ in $[1..K]$. 
This effect is illustrated in \autoref{fig:concept} a and c.  
Hierarchical CPI can infer that $P$ is important, but will give little to no importance to $L$ and $R$ (and all downstream groups). 
% \todo{This figure relies on color as the sole cue, it is dangerous as people can be variable good at reading color. I wonder why the p value distributions are so similar. Is importance preservation then actually meaningful if the single-variable importance values is the same?DE}
%
To obtain a Shapley-like additive decomposition of the model fit, inspired by variance partitioning ideas \cite{lescroart2015fourier}, we introduce a transfer mechanism that ensures additivity at the node level, meaning that the importance of a parent node is split into the sum of the importance of its children.
This is meant to allow a more refined allocation of the importance budget compared to traditional clustering methods. 
For a node $R$, let $\tilde{\psi}_{R}^k$ be the corrected value for ${\psi}_{R}^k$ that ensures importance conservation through the hierarchical structure. 
Then, $\mathds{1}_{R}(\epsilon)$ is the indicator function equal to one when $\psi_{R} / \hat\sigma_R \geq \epsilon > 0$, where $\hat\sigma_R$ is the standard deviation of the right node importance estimated over the k-folds and zero otherwise. 
Importance conservation aims at satisfying the equation $\tilde{\psi}_P^k = \tilde{\psi}_{L}^k + \tilde{\psi}_{R}^k$.
This condition ensures that the importance of the top node, which measures the full model's importance, is allocated to its children nodes.
The allocation mechanism for the child node $L$ with sibling $R$ and parent $P$ proceeds as follows:
\begin{align}
\label{eq:imp_conserv}
\tilde{\psi}_{L}^k=
\begin{cases}
    \psi_{L}^k + \mathds{1}_{R}(\epsilon) \frac{\tilde{\psi}_P^k - \psi_{L}^k - \psi_{R}^k}{2} & \text{if } \mathds{1}_{L}(\epsilon) = 1\\
        \tilde{\psi}_P^k \frac{\psi_{L}^k}{\psi_{R}^k + \psi_{L}^k} (1- \mathds{1}_{R}(\epsilon)) + \mathds{1}_{R}(\epsilon) (\tilde{\psi}_P^k - \psi_{R}^k) & \text{if } \mathds{1}_{L}(\epsilon) = 0
\end{cases}
\end{align}
This equation distributes the parent's importance proportionally to the children's importance when both nodes' importance values are greater than a threshold $\epsilon$.
If the importance $\psi_L$ of node $L$ is smaller than $\epsilon$, it remains unchanged, avoiding false positives (not all children of an important parent are important).
When both children have sub-threshold importance, indicating that their correlation leads to mutual importance cancellation, the importance of the parent is allocated equally between the two nodes.  
\section{Results}

\subsection{Control of Family-Wise Error Rate on Simulated Data}
\label{subsec:simulation}
\begin{figure}[h!]
\includegraphics[width=\textwidth]{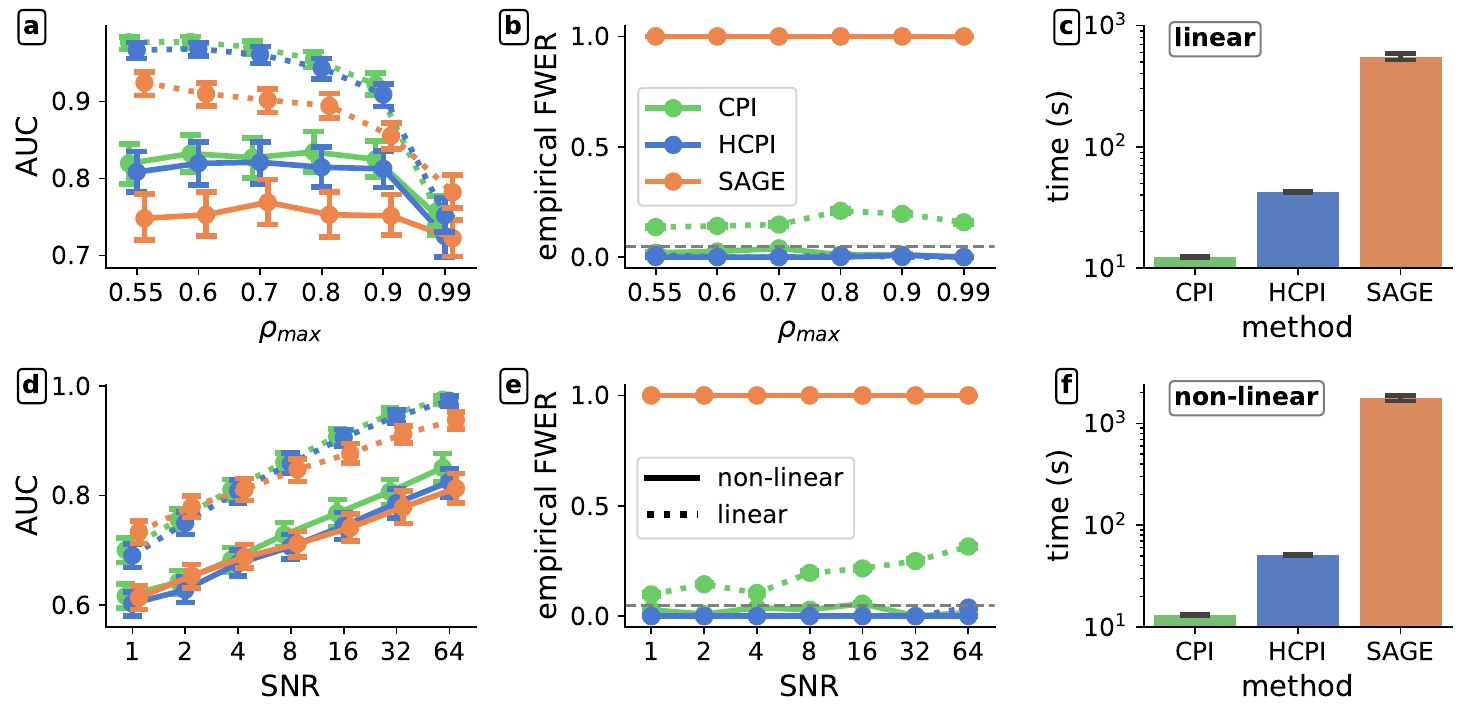}
\caption{
\textbf{Hierarchical CPI empirically controls the FWER wiht high statistical power.} 
Results from simulated data with $n=400$ samples and $p=124$ variables sampled from a normal distribution with block correlations described in \autoref{subsec:simulation}.
The results present a summary of 100 repetitions of the simulation. 
\textbf{a} and \textbf{d} present the AUC for important/non-important variable classification as a function of correlation and SNR. 
Error bars represent 95\% confidence interval.
\textbf{b} and \textbf{c} show the evolution of the FWER at level $\alpha=0.05$ with Bonferroni correction. %
The top row explores varying $\rho_{max}$ values at a fixed SNR=2, while the bottom row examines varying SNRs at a fixed $\rho_{max}=0.9$.
\textbf{c} and \textbf{f} shows the average computation time taken by each method over the simulations for the linear and non-linear scenario respectively.  
}
\label{fig:simulation}
\end{figure}
%
% We generalize the notation $S_0$ to any group in the hierarchical tree that only contains null variables. 
% %
% We denote $\mathcal{G}$ the set of all groups in the hierarchical tree , $S_0 = \{G \;|\; \; G\in\mathcal{G},\;\forall j\in G, \; \beta_j =0 \}$. For any significance level $\alpha$, we denote the estimated support $\hat S = \{ G \;|\; p^h_G \leq \alpha \}$.
This experiment benchmarks the hierarchical-CPI approach described in Algorithm \autoref{alg:hcpi} with other state of the art variable importance methods on simulated data. 
The data is generated by blocks, each corresponding to an AR(1) with autocorrelation parameter
$\rho_{\text{max}}$.
The outcome is modeled using two different scenarios. 
The first is a linear model with additive noise $y =  X \beta+ \sigma_N \varepsilon$, represented with dotted lines in \autoref{fig:simulation}.
The support $S^* = \{j ; | ; \beta_j \neq 0\}$ is kept sparse, with $|S^*|$ set to either 5 or 10, and coefficients values sampled from the set $\beta_j \in \{-2, -1, 1, 2\}$ with uniform probability.
The second is a non linear scenario presented in \cite{chamma_2024_statistically}: $y=X_{j_1} + 2\log(1+ 2X_{j_2}^2 + (X_{j_3} +1)^2) + X_{j_4}X_{j_5} + \sigma_N\epsilon$.
The support $S^*=\{j_1, \cdots, j_5\}\in [\![ 0, p ]\!] ^5$ is randomly sampled at each simulation run.
In both cases, additive noise $\varepsilon \sim \mathcal{N}(0, I_n)$, controls the signal-to-noise ratio (SNR), defined as $SNR = ||y^*||_2^2 / \sigma_N^2 ||\epsilon||_2^2$, where $y^*$ is the noiseless outcome. 
The SNR is a simulation parameter.
In the experiments shown in \autoref{fig:simulation}, we used $p=124$ variables grouped into five blocks of correlated features with respective sizes 4, 8, 16, 32, and 64.
The number of samples was fixed at $n=400$ to match the dimensionality commonly found in medical imaging applications.
To solve the non-linear regression task, a multilayer perceptron (MLP) with 100 hidden units was used. 
It was trained using \textit{Adam} optimizer for 400 epochs with early stopping (patience 10 epochs). 
For the linear scenario, a Ridge-regularized model was used. 
Its regularization parameter was learned via nested cross-validation.
For CPI and HCPI, the loss used in \autoref{eq:CPI} is the root mean squared error.  

The top row (\textbf{a, b, c}) explores the influence of $\rho_{max}$ while the bottom row the influence of the SNR. 
We considered two metrics. 
First, the Area Under the Receiver Operating Characteristic Curve (AUC): This compares the predicted importance to the true importance.
It aims to assess each method's ability to recover the true support.
Second, the FWER: This measures the probability of making at least one false discovery, estimated over 100 simulation repetitions.
We compare three methods: CPI, Hierarchical-CPI (HCPI) and SAGE. 
For CPI and HCPI, the predicted importance correspond to $1-p$-value and the estimated support is $\hat S_\alpha = \{ j \; | \; p_j \leq \alpha\}$, we considered a level $\alpha=0.05$. 
Regarding SAGE, we used a publicly available implementation\footnote{https://github.com/iancovert/sage} which provides estimated standard deviation from which 95\% confidence interval were derived.
The estimated support for SAGE consisted of variables for which the confidence interval did not include 0.

As shown in \autoref{fig:simulation}, the hierarchical approach effectively controls the FWER even in challenging simulation settings, e.g. with very high correlation or low SNR.
This can be attributed to the hierarchical adjustment, described in \autoref{eq:hierarchical_adjust}, that bounds a node's p-value below using the  p-value of its parent.
As shown in subpannels \textbf{a} and \textbf{d} of the \autoref{fig:simulation}, this additional control does not decrease the power of the method when compared to CPI. 
While the exploration of the nodes entails an additional computation cost, it remains of the same order as CPI, which is a fast method. 
Indeed, for a hierarchical clustering of $p$ variables, the total number of nodes is $2p -1$ which makes the computation scale linearly with the dimension instead of the exponential explosion  inherent to SAGE \cite{covert_2020_sage}. 
This fact is illustrated on the panels \textbf{c} and \textbf{f} of \autoref{fig:simulation}, where the logarithmic axis illustrates the untractable computation time of SAGE. 
In the non-linear case where a neural network is used, this trend ibecomes even more pronounced. 

\subsection{Hierarchical CPI Identifies Characteristic Markers of AD}
\begin{figure}[h!]
\includegraphics[width=\textwidth]{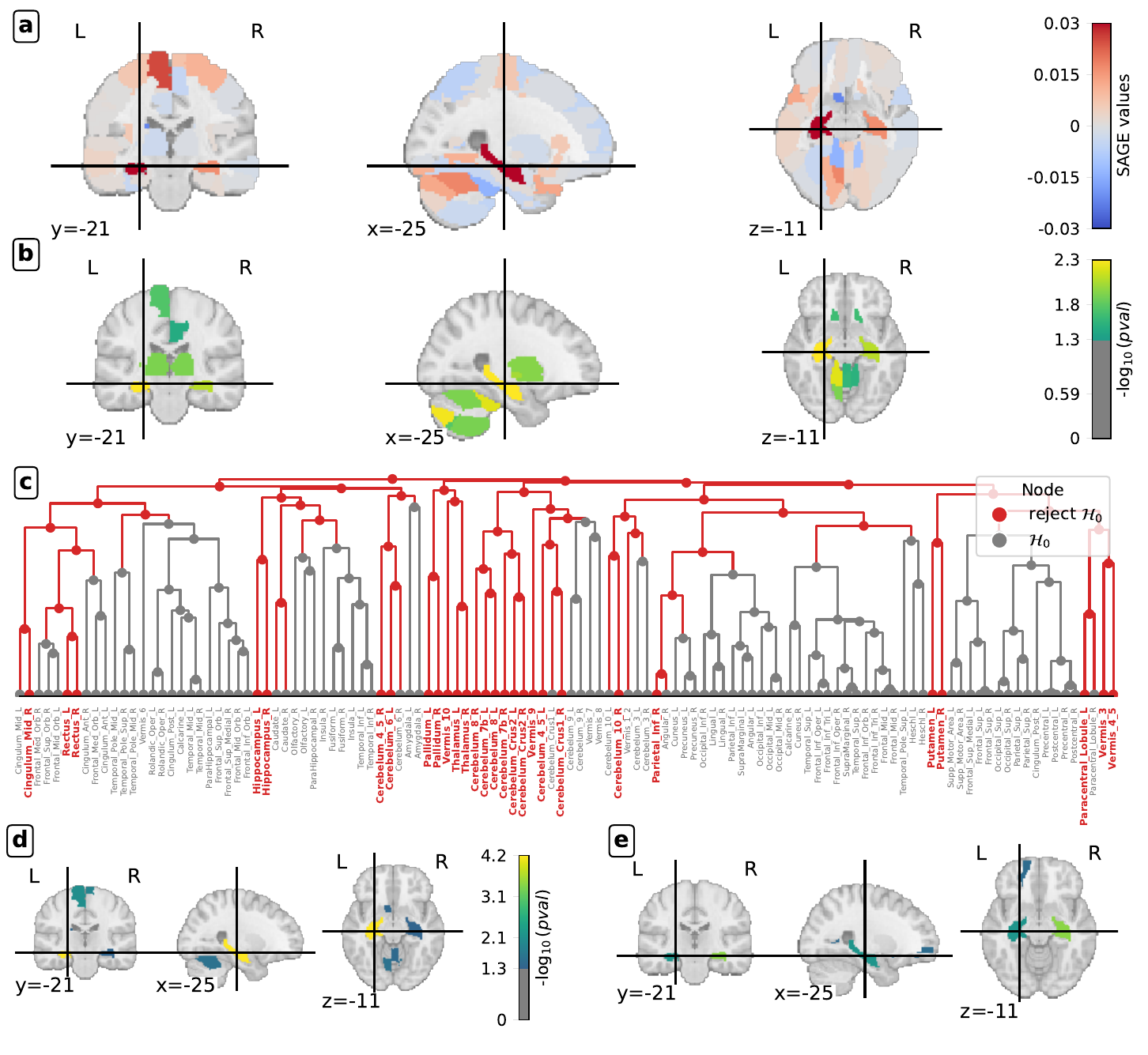}
\caption{
\textbf{Hierarchical CPI discovers groups of characteristic markers of AD progression.}  
Importance obtained for classifying AD and MCI subjects from the ADNI dataset using a support vector classifier and grey matter densities in the 116 AAL regions.
\textbf{a} Signed importance values obtained using the SAGE method. 
Only values for which the 95\% CI did not overlap with 0 are reported.
\textbf{b} Important regions identified by hierarchical CPI at the $\alpha=0.05$ level.
\textbf{c} Dendrogram derived from the hierarchical CPI approach.
The important nodes ($\alpha=0.05$) of the tree learned by hierarchical clustering are colored in red.
Regions identified as important are labeled in red.
\textbf{d} Important regions identified by HCPI for classifying patients with AD versus controls.
\textbf{e} Important regions for classifying controls versus patients with MCI.
}
\label{fig:ADNI_sage}
\end{figure}

This study explores image-based diagnosis using the ADNI dataset \cite{petersen_2010_alzheimer}. 
Cohort selection was based on the availability of T1-weighted images, similarly to \cite{samper_2018_reproducible}. 
A total of 1616 patients were included: 760 controls (CN), 529 diagnosed with Mild Cognitive Impairment (MCI), and 327 with AD.
Gray Matter (GM) density maps were computed using the \textit{sMRIPrep} pipeline \cite{esteban_2021_smriprep}, which is part of the widely used \textit{fMRIPrep} pipeline \cite{esteban_2019_fmriprep}. 
The mean GM densities were extracted from 116 Regions of Interest (ROIs) defined by the Automated Anatomical Labeling Atlas 3 \cite{rolls_2020_automated} using \textit{Nilearn} \cite{Nilearn} and used as features for classification tasks.
The high correlation between ROIs (Pearson correlation ranging from 0.32 to 0.94) and the interest in locating the pathology's impact precisely motivated the proposed approach. 
The methodology for data processing, model optimization, and hyper-parameter tuning followed \cite{samper_2018_reproducible}, to ensure result reproducibility.
We used a Support Vector Classifier (SVC) as implemented in \textit{libsvm} on GM densities in the 116 regions of the AAL atlas.
For the HCPI method, importance was measured by the hinge loss difference in \autoref{eq:CPI}. 
All results are reported using 10-fold cross-validation with stratification. 
Hyper-parameter tuning was performed using a nested cross-validation loop to avoid information leakage from the test set \cite{pedregosa_2011_scikit}.
 Using a linear or \textit{rbf} kernel led to similar predictive performance and importance scores. 
The results were reported for a linear kernel. 
Three classification tasks were considered: MCI vs. AD, AD vs. CN, and MCI vs. CN. 
The average AUC on the test set over 10 folds were 0.78, 0.93, and 0.74, respectively, which is consistent with existing literature \cite{samper_2018_reproducible}.
\autoref{fig:ADNI_sage} presents the importance maps at the individual feature resolution, computed using the SAGE method (\textbf{a}) and the proposed hierarchical-CPI (\textbf{b}) for classifying patients with AD and MCI.
Similar to the previous section, for both methods, results are reported at a significance threshold of $\alpha=0.05$.
While SAGE identifies more regions as important, it is likely that many of these are only marginally, not conditionally, associated with the outcome. 
By contrast, HCPI identifies fewer regions that summarize the specific markers of the disease. 
Importantly, all regions identified by HCPI are also identified by SAGE, revealing a form of consistency.
These regions include the hippocampi, and the orbitofrontal cortex (rectus in the AAL) which have been extensively described in the literature as areas where atrophy is substantial \cite{knopman_2021_alzheimer, van_2000_orbitofrontal}.
The putamen and thalamus were also identified as predictive, consistently with published work documenting the association between AD and decreased global GM in these regions \cite{de_2008_strongly}.
%
% In addition, HCPI casts the problem of inference as identifying groups of variables that are jointly predictive.
%
% This approach allows for learning and discovering clusters of predictive variables with varying resolutions, beyond merely testing individual variables or predefined groups.
%
Moreover, HCPI allows to learn clusters of predictive variables with varying resolutions,
This is depicted in \textbf{c}, which presents the dendrogram learned by agglomerative clustering, with nodes having a p-value below the significance threshold highlighted in red.
This information complements panel \textbf{b} by highlighting the importance of selected subgroups.
%
% This information is richer than a simple measure of importance at the variable level, as it allows for the identification of clusters of variables that are important as a group, even if none of the individual variables are significant.
%
For instance, the node including (\texttt{Caudate\_L, Caudate\_R}) is important whereas individual variables are not, due to the high correlation between these two regions, (Pearson correlation of 0.84) leading to a cancellation of their conditional importance.
In addition to these results \autoref{fig:ADNI_sage} presents the importance map for two additional tasks: AD vs CN (\textbf{d}) and MCI vs CN (\textbf{e}). 
Similar to \autoref{fig:ADNI_sage}, the HCPI approach identifies hallmarks of AD pathology, such as the gray matter density in the hippocampi, which has been extensively described in the literature \cite{knopman_2021_alzheimer}.

\subsection{Importance Conservation Enables Inference on Highly Correlated EEG Data}
\begin{figure}[ht]
\includegraphics[width=\textwidth]{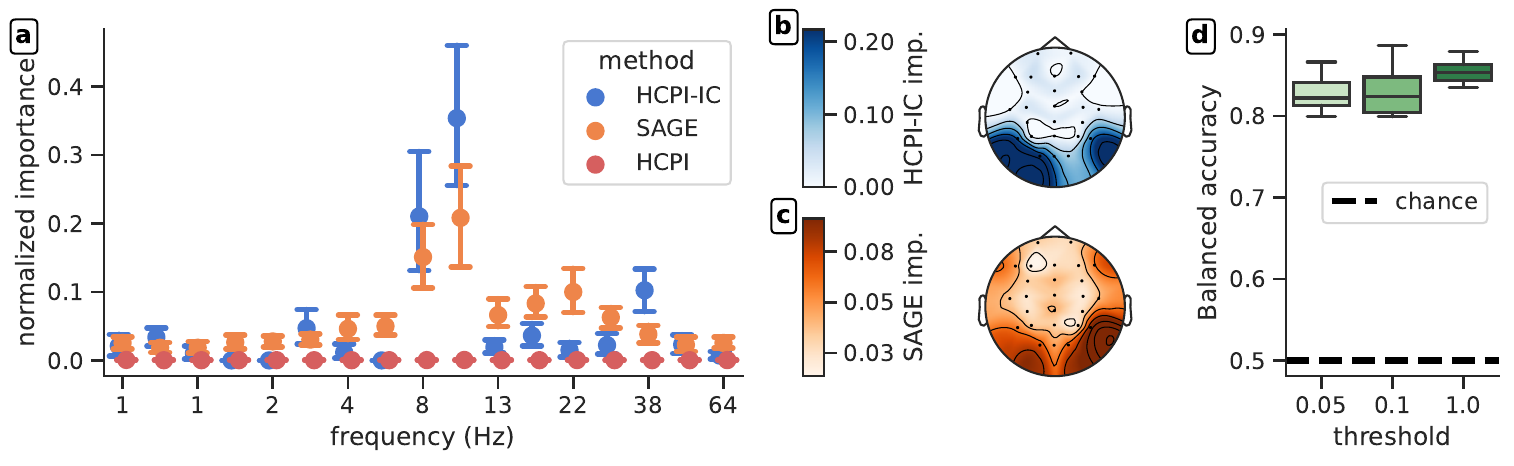}
\caption{\textbf{
Importance conservation enables variable importance inference in high-dimensional settings with very high correlation.}
Comparison of the variable importance obtained using hierarchical CPI with importance conservation (HCPI-IC), without (HCPI) and using SAGE. 
Absolute SAGE values are represented for readability.
\textbf{a} The distribution of importance over frequencies for significant variables at the $\alpha=0.05$ threshold.
Each point represents the sum of important channels at a given frequency.
\textbf{b, c} The distribution of importance over scalp topography for significant variables at the $\alpha=0.05$ threshold.
At a given channel, the sum is taken over all important frequencies.
\textbf{d} Presents the performance of sub-models that use only a fraction of variables identified as significant at a level $\alpha$, with $\alpha=1$ corresponding to the full model.
The boxes represent the distribution over 10-fold cross-validation. 
The dotted line represents the chance level.
} \label{fig:eeg}
\end{figure}
The importance conservation approach was then applied to the EEG data from the TDBRAIN dataset to characterize the known Berger effect \cite{van_2022_two}.
EEG data is known to exhibit high correlation due to latent sources spreading across the scalp as a result of field spreads.
Resting state EEG were acquired from 1234 healthy subjects who were asked to open and close there eyes during labeled periods.
The dataset was preprocessed using the pipeline presented in \cite{bomatter_2023_machine} in order to remove artifacts generated by non-brain sources. 
Specifically, independant component analysis was applied in order to remove eye-movement artifacts which would make the task trivial. 
The power at each of the 26 electrodes was computed across 17 logarithmically spaced frequency bands, ranging from 1 to 64 Hz, using Morlet wavelets.
The 442 resulting features present a very high correlation structure, with minimum Pearson correlation above 0.9. 
We considered the task of classifying the eyes status (closed vs open) using a pipeline consisting of a logarithm computation followed by logistic regression. 
Similarly to the previous section, 10-fold nested cross validation was used. 
The loss used to measure the conditional importance in \autoref{eq:CPI} is the cross-entropy loss.
The transmission threshold $\epsilon$ was set to the 95\% quantile of the normal distribution. 

As illustrated by \autoref{fig:eeg} \textbf{a}, the high correlation in the data causes the conditional importance to vanish, resulting in no significant discoveries at a threshold of $\alpha=0.05$ (HCPI in red).
To mitigate this issue and increase statistical power, the importance conservation mechanism (HCPI-IC, in blue) introduced in \autoref{subsec:importance_conservation} is applied.
The importance however remains more focal than with SAGE (orange), which spreads the importance over a wide range of frequencies.
Panels \textbf{a}, \textbf{b} and \textbf{c} show the distribution of importance in the frequency and sensor spaces among significant variables at a threshold of $\alpha=0.05$.
The pattern observed, with most of the importance located around 10Hz at occipital electrodes, corresponds to the well-studied Berger effect, characterized by increased occipital activity in the alpha-band \cite{hohaia_2022_occipital}. 
The pattern is precisely identified by HCPI-IC, while SAGE distributes more broadly the importance over electrodes and frequencies. 
Finally, panel \textbf{d} demonstrates the performance of submodels using only significant variables at thresholds  $\alpha<0.05$, $\alpha<0.1$, and 1 (all variables).
It reveals that at the strictest threshold ($\alpha = 0.05$), the procedure selects 55 variables out of 442, recovering 96\% of the full model's performance.
\section{Discussion and conclusion}
%
% previously conclusion
The HCPI approach was motivated by the challenge of making inference on high-dimensional and highly correlated neuroimaging data.
To achieve this, it frames the inference problem as the discovery of groups of variables that are jointly predictive of the outcome. 
It can recover statistical control in high correlations regimes where standard methods lose consistency.
% previously the begining of Discussion that was redundant with ccl
% We have introduced HCPI, a model-agnostic algorithm for discovering groups of independent variables in highly correlated medical imaging data.
%
Statistical guarantees were empirically validated on simulated data, and the method was applied to two neuroimaging modalities using publicly available datasets.
Its effectiveness was demonstrated on both classification and regression tasks. 
By successfully testing different tasks, models and losses, we proved the practical utility of this model-agnostic approach.
HCPI flexibility exceeds that of existing methods relying on linear models or Lasso-based knockoffs \cite{mandozzi_2016_hierarchical, chevalier_2021_decoding, binh_2019_ecko}.

By exploring subgroups within a learned hierarchical tree, HCPI balances precision and statistical power, allowing the identification of groups that are important, even if none of the individual variables is significant. 
It thus identifies the highest resolution at which importance can be narrowed down, without needing to optimize clustering parameters \cite{chevalier_2021_decoding, chamma_2024_variable, binh_2019_ecko}.
This information can easily be visualized using a dendrogram.
Unlike additive methods like SAGE, which exhaustively explore all subgroups including a variable, eliminating many costly and useless evaluations.
It provides a FWER control, thus contrasting with SAGE's known lack of type-1 error control \cite{verdinelli_2024_feature}. 
Moreover, it remains tractable, requiring only $2p-1$ importance evaluations.
Finally, the importance conservation mechanism introduced in \autoref{eq:imp_conserv} mitigates power loss due to vanishing importance, as shown with highly correlated EEG data features.

When tested on MRI and EEG data, our method identified biologically well-studied features consistent with existing literature such as hallmarks of AD in MRI and the Berger effect in EEG. 
Lastly, while we used Conditional Permutation Importance, because known to be more stable and efficient than LOCO, the latter could also be used as a drop-in replacement for estimating importance. 
\paragraph{Limitations:}
We have explored scenarios where a single agglomerative clustering is performed, demonstrating that it can yield insightful learnings about data structure, with clusters of variables being predictive even if individual variables are not. 
However, this step can introduce randomness. 
For applications not requiring hierarchical tree learning, like voxel-level or applications to raw images—it may be beneficial to repeat the procedure and leverage p-value aggregation strategies or e-values \cite{meinshausen_2009_p, vovk_2021_values}.
Future work could involve repeated agglomerative clustering on random data subsets, followed by aggregation to improve robustness.
%
%\textcolor{red}{
Another limitation concerns the theoretical guarantees of the importance conservation approach.
While we empirically observed a type-1 error rate much lower than SAGE, a formal result remains to be established. 
The transmission threshold $\epsilon$ is critical in this context:
it defines a threshold below which the importance of a parent node becomes indivisible because the contributions of its children nodes cancel each other out.
We conjecture that it is possible to obtain guarantees for type-1 error control outside a neighborhood (which size depends on $\epsilon$) around the support.

% \paragraph{Conclusion:}
%

% Our proposed approach aims to provide actionable insights from complex data without sacrificing statistical guarantees.
%
% We believe that maintaining these statistical guarantees is crucial for the widespread adoption of variable importance methods.
%
The algorithm builds on open-source software available on Github\footnote{https://github.com/mind-inria/hidimstat}.
\begin{credits}
\subsubsection{\ackname} This esearch has received funding from the H2020 Research Infrastructures Grant EBRAIN-Health 101058516 and the VITE ANR-23-CE23-0016 and  PEPR Santé numérique, Brain health Trajectories ANR-22-PESN-0012 projects.
\end{credits}

% ---- Bibliography ----
%
% BibTeX users should specify bibliography style 'splncs04'.
% References will then be sorted and formatted in the correct style.
%
% \bibliographystyle{splncs04}
\printbibliography
\end{document}